\documentclass[11pt]{article}

\usepackage[utf8]{inputenc}
\usepackage[T1]{fontenc}
\usepackage{lmodern}
\usepackage{amsmath,amssymb,amsthm,amsfonts}
\usepackage{mathtools}
\usepackage{microtype}
\usepackage{enumitem}
\usepackage{geometry}
\usepackage{hyperref}
\usepackage{bm}
\usepackage{dsfont}
\usepackage{tikz-cd}

\usepackage{mathrsfs}
\usepackage{bbm}

\newcommand{\E}{\mathbb{E}}

\setlist{nosep}

\newtheorem{theorem}{Theorem}[section]
\newtheorem{proposition}[theorem]{Proposition}
\newtheorem{lemma}[theorem]{Lemma}
\newtheorem{corollary}[theorem]{Corollary}
\theoremstyle{definition}
\newtheorem{definition}[theorem]{Definition}

\theoremstyle{remark}
\newtheorem{remark}[theorem]{Remark}

\newcommand{\RR}{\mathbb{R}}

\newcommand{\PP}{\mathbb{P}}
\newcommand{\EE}{\mathbb{E}}
\newcommand{\1}{\mathds{1}}
\newcommand{\eps}{\varepsilon}

\title{Psychometric Tests for AI Agents and Their Moduli Space}
\author{Przemyslaw Chojecki \\ \small ulam.ai}
\date{November 24, 2025}

\begin{document}
\maketitle

\begin{abstract}
We develop a moduli-theoretic view of psychometric test batteries for AI agents and connect it explicitly to the AAI score developed in \cite{AAIscore}. First, we make precise the notion of an AAI functional on a battery and set out axioms that any reasonable autonomy/general intelligence score should satisfy. Second, we show that the composite index ('AAI-Index') defined in \cite{AAIscore} is a special case of our AAI functional. Third, we introduce the notion of a cognitive core of an agent relative to a battery and define the associated AAI$_{\textrm{core}}$ score as the restriction of an AAI functional to that core. Finally, we use these notions to describe invariants of batteries under evaluation-preserving symmetries and outline how moduli of equivalent batteries are organized.
\end{abstract}

\section{Introduction}\label{sec:intro}
In psychology psychometric evaluation aims to measure capability via structured batteries of tasks. In large-agent settings (LLMs, tool-using systems, embodied agents), a battery can be understood as structured data plus evaluation rules plus resource accounting. This paper builds on the AAI score \cite{AAIscore} by: (i) formalizing batteries as objects and morphisms, (ii) defining AAI functionals on batteries, (iii) proving that the AAI-Index from AAI is an instance of such a functional, (iv) introducing AAI$_{\textrm{core}}$, a score attached to the agent's cognitive core, and (v) describing the moduli space of batteries and its continuous structure.

In this framework, AGI is not a single-task milestone but a property of an agent's performance over the moduli of batteries. A frequent critique of AI evaluation is the "overfitting" to specific benchmarks. \emph{This paper argues mathematically that one need not obsess over the specifics of the tests themselves, provided one samples enough of them.} This is akin to students training for a Math Olympiad by solving problems from past competitions: they saturate the space of known problem types not to memorize answers, but to internalize the underlying logic required to solve new, similar problems.

Our determinacy result (Theorem \ref{thm:dense-determinacy}) implies a practical certification rule: if the family of scoring functionals is "regular" (Lipschitz continuous on the moduli space), then performance on a sufficiently dense finite panel (a $\delta$-net of canonicalized batteries) mathematically certifies the agent's performance across the entire continuum of possible tests, up to a controllable error term.

Specifically, if an agent's AAI score meets or exceeds the AGI threshold on every panel instance with margin $m$, then its worst-case AAI over the whole moduli is at least threshold $-2L\delta$ (where $L$ is the Lipschitz modulus). In other words, sufficiently many strong results on a well-designed, diverse panel certify the level globally. Because the functional is resource-aware (via the projective geometry of resource rays and explicit cost terms), a certified level is about capabilities per unit cost, not just raw wins, aligning with the operational AAI score \cite{AAIscore}.

A second insight is that the cognitive core makes “true generality” testable rather than anecdotal. High $\mathrm{AAI_{\textrm{core}}}$ across heterogeneous task families, small gaps between the full AAI and $\mathrm{AAI_{\textrm{core}}}$ (indicating competence rather than scaffolding), low dispersion penalties, and stability under battery symmetries and seeded drifts together provide necessary evidence patterns for AGI. Moreover, when the induced core aligns with broad CHC-style factors and remains low-dimensional while still reconstructing threshold decisions across families, it indicates transferable structure rather than overfit tricks \cite{hendrycks-agi}.

Many of our structural results are not specific to the particular moduli space of test batteries. In fact, the determinacy, regularity, and concentration theorems hold for any Lipschitz-regular functionals on a metric space $(S,d)$ with a dense family of "simple" configurations and, for the axis-based scores, a finite collection of axis functionals whose expectations are Lipschitz in $d$.

\subsection*{Notation and conventions}
For a measurable space $X$, let $\mathcal{P}(X)$ denote probability laws on $X$. For task families $T=\bigsqcup_k F_k$, write $u(t)\in[0,1]$ for a PIT-uniformized per-task score (defined below), $q(t)\in[0,1]$ for threshold-aligned success, and $r\in\RR_+^{d_R}$ for resources. Expectations $\EE[\cdot]$ are taken over seeds/drifts unless stated otherwise.

\section{Batteries and Morphisms}

\begin{definition}[Battery]\label{def:battery}
A \emph{battery} is an octuple
\[
  \mathcal B=(T,\ \mathcal F,\ \mathsf S,\ Q^*,\ \mu,\ \mathsf D,\ \Pi,\ \mathsf R),
\]
where:
\begin{itemize}
  \item $T$ is a finite set of tasks; $\mathcal F=\{F_k\}$ is a partition of $T$ into families.
  \item $\mathsf S=\{S_t:\Omega_t\to[0,1]\}_{t\in T}$ are task-specific scoring maps.
  \item $Q^*:T\to[0,1]$ are task thresholds.
  \item $\mu$ is a sampling law on $T\times\Pi\times\mathsf D$ (tasks, seeds, drifts).
  \item $\mathsf D$ (drifts) and $\Pi$ (seeds) are measurable spaces.
  \item $\mathsf R\cong\mathbb{R}^{d_R}$ are resource coordinates (e.g., time, tokens, cost), recorded nonnegatively.
\end{itemize}
\end{definition}

\begin{definition}[Agent representation on a battery]\label{def:agent-rep}
Fix $\mathcal B$ and an agent $\mathcal A$. A run on $t\in T$ with seed $s\in\Pi$ and drift $\delta\in\mathsf D$ and resource vector $r\in\mathsf R$ produces a score $q(t)=S_t(\mathrm{Run}(\mathcal A;t,s,\delta,r))\in[0,1]$. Write
\[
  X_{\mathcal B}:=[0,1]^T\times \mathbb{R}_{\ge 0}^{\,d_R}.
\]
Let $\mathbb{P}$ be the joint probability over $(t,s,\delta)$ drawn from $\mu$ and any internal randomness of $\mathcal A$. The \emph{agent representation} is the pushforward (image) measure
\[
  \rho_{\mathcal B}(\mathcal A)
  \ :=\ \mathbb{P}\circ\big((q(t))_{t\in T},\, r\big)^{-1}
  \ \in\ \mathcal{P}(X_{\mathcal B}).
\]
\end{definition}

\begin{definition}[Morphisms of batteries]\label{def:morphism}
A morphism $f:\mathcal B\to\mathcal B'$ is a tuple
\[
  f=(f_T,\ f_\Pi,\ f_\mathsf D,\ f_\mathsf R,\ \{\phi_t\}_{t\in T})
\]
satisfying:
\begin{enumerate}
  \item $f_T:T\to T'$ respects families ($t\in F_k\Rightarrow f_T(t)\in F'_k$).
  \item $f_\Pi:\Pi\to\Pi'$ and $f_\mathsf D:\mathsf D\to\mathsf D'$ are measurable and measure-preserving relative to $\mu$.
  \item $f_\mathsf R:\mathsf R\to\mathsf R'$ is linear and unit-consistent.
  \item $\phi_t:[0,1]\to[0,1]$ are strictly increasing and continuous.
  \item (Evaluation preservation) For all agents $\mathcal A$ and $t\in T$,
  \[
    S'_{f_T(t)}(\mathrm{Run}(\mathcal A;f_T(t),f_\Pi(s),f_\mathsf D(\delta),f_\mathsf R(r)))
    \ =\ \phi_t\!\left(S_t(\mathrm{Run}(\mathcal A;t,s,\delta,r))\right),
  \]
  almost surely (w.r.t.\ $\mu(\cdot\mid t)$).
\end{enumerate}
\end{definition}

\begin{proposition}\label{prop:category}
Batteries and their morphisms form a category $\mathbf{Bat}$ under composition and identities.
\end{proposition}

\begin{proof}
Identity: take identity maps with $\phi_t=\mathrm{id}$. Composition is componentwise; evaluation preservation is stable under composition. Associativity and identity laws are immediate.
\end{proof}

\section{Symmetries and the Moduli Space}

\begin{definition}[Symmetry group]\label{def:G}
Let
\[
  G\ :=\ \Big(\textstyle\prod_k \mathrm{Sym}(F_k)\Big)\ \times\ 
  \mathrm{Aut}_\mu(\Pi)\ \times\ \mathrm{Aut}_\mu(\mathsf D)\ \times\
  \Big(\textstyle\prod_{t\in T}\mathrm{Homeo}^+([0,1])\Big)\ \times\
  (\mathbb{R}_{>0})^{d_R},
\]
acting by within-family task permutations, measure-preserving relabelings of seeds and drifts, strictly increasing per-task score reparameterizations, and positive unit rescalings of resources.
\end{definition}

\begin{definition}[Moduli space]\label{def:moduli}
The \emph{moduli space of batteries} (coarse, set-level) is the set of isomorphism classes
\[
  \mathfrak M\ :=\ \mathbf{Bat}/G.
\]
\emph{Remark.} If automorphisms matter, the appropriate object is the quotient stack (action groupoid) $[\mathbf{Bat}/G]$. Here we work with $\mathfrak M$.
\end{definition}

\section{PIT Normalization and Canonical Representation}\label{sec:PIT}
\begin{definition}[Randomized PIT]\label{def:PIT}
Given a scalar task score $s(t)$ with conditional CDF $F_t$, define $u(t)=F_t(s(t)^-) + \xi\big(F_t(s(t)) - F_t(s(t)^-)\big)$ with $\xi\sim\mathrm{Unif}[0,1]$. Then $u(t)\sim \mathrm{Unif}[0,1]$ conditional on drift and seed.
\end{definition}

\begin{proposition}[Uniformity and invariance]\label{prop:PIT-inv}
PIT-normalized $u(t)$ is invariant to strictly increasing reparameterizations of $s(t)$ and enables copula-based comparisons across tasks.
\end{proposition}
\begin{proof}
Uniformity follows from the probability integral transform with randomization at discontinuities. We define the transformed CDF
\[
F_t^{\phi}(y)
\;:=\;
\mathbb{P}\bigl(\phi(s(t)) \le y \,\big|\, \text{drift, seed}\bigr).
\]
If $\phi$ is strictly increasing, then
\[
F_t^{\phi}(y)
= \mathbb{P}\bigl(\phi(s(t)) \le y \,\big|\, \dots\bigr)
= \mathbb{P}\bigl(s(t) \le \phi^{-1}(y) \,\big|\, \dots\bigr)
= F_t\bigl(\phi^{-1}(y)\bigr).
\]
In particular, for $y = \phi(s)$ we have
\[
F_t^{\phi}\bigl(\phi(s)\bigr) = F_t(s),
\]
so the distribution of $u(t)$ is unchanged; the copula of $(u(t))$ captures dependence invariantly. See \cite{ruschendorf09} for a proof.
\end{proof}

\begin{definition}[Canonical representation]\label{def:canonical}
Let $X^{can}$ contain $(u(t))_{t\in T}$, threshold-aligned $q(t)=\1\{u(t)\ge \tau(t)\}$ with $\tau(t)=F_t(Q^*(t))$, and the resource ray $[r]$. The pushforward law on $X^{can}$ is the canonical representation of an agent on $\mathcal B$.
\end{definition}

\section{Topological Structure of the Moduli}\label{sec:topology}

\begin{proposition}[Continuous parameters]\label{prop:continuous}
Fix the discrete task skeleton (families, anchors) and threshold structure.
The moduli space decomposes as:
\[
\mathfrak{M} \simeq
\text{(Discrete Data)} \times
\bigl(\text{Thresholds } \tau \in [0,1]^T\bigr) \times
\bigl(\text{Copulas } C_u \in C \bigr) \times
\bigl(\text{Rays } [r] \in \mathbb{P}(\mathbb{R}^{d_R}_+)\bigr),
\]
where $C$ denotes the set of all copulas on $[0,1]^T$, endowed with the $W_1$-topology.
\end{proposition}
\begin{proof}
Discrete invariants (family sizes, anchor structure, threshold ordering) partition the space into strata.
Within a fixed stratum, every battery is specified by the coordinates $(\tau, C_u, [r])$: the threshold vector $\tau$, the copula $C_u$ of the canonical PIT scores, and the projective resource ray $[r]$.
The symmetry group $G$ acts trivially on $\tau$ (once the discrete threshold structure is fixed), while $C_u$ varies continuously in the Wasserstein topology and $[r]$ in the projective topology. This yields the stated decomposition.
\end{proof}

\begin{remark}[Wasserstein geodesics]\label{rem:geodesics}
Viewing each copula $C_u$ as a probability measure on $[0,1]^T$ and equipping
the copula space with the Wasserstein--1 metric $W_1$, any two copulas
$C_0,C_1$ can be connected by a $W_1$-geodesic $(C_t)_{t\in[0,1]}$, also called
a displacement interpolation. The path $t\mapsto C_t$ describes a controlled
drift of dependence across tasks: $C_t$ deforms $C_0$ into $C_1$ by moving
probability mass along optimal transport plans. Evaluating our AAI functionals
along such geodesics provides a natural way to stress test robustness against
gradual changes in task dependence. See \cite{villani}.
\end{remark}

\paragraph{Discrete and continuous moduli.}
Fixing a battery $\mathcal B$ determines a collection of discrete
invariants: the partition of tasks into families $\{F_k\}$, the choice
of anchors within each family, and the qualitative threshold structure
(e.g.\ which thresholds coincide and the induced partial order on
families).  We refer to this finite combinatorial datum as the
\emph{battery skeleton} and denote it by $S$.  Let $\mathcal S$ be the
set of all such skeletons, modulo the action of the symmetry group $G$.
For a fixed skeleton $S\in\mathcal S$, Proposition~\ref{prop:continuous}
shows that the corresponding configurations form a continuous parameter
space
\[
\mathfrak M_S
\simeq
\bigl\{\tau\in[0,1]^T\bigr\}
\times
\bigl\{C_u \in C\bigr\}
\times
\mathbb P(\mathbb R^{d_R}_+),
\]
where $\tau$ collects the task thresholds, $C_u$ is the copula of the
canonical PIT scores, and $[r]\in\mathbb P(\mathbb R^{d_R}_+)$ is the
resource ray.  In particular, once $S$ is fixed, the remaining degrees
of freedom are purely continuous and live in a product of metric and
measure-theoretic spaces (thresholds in $[0,1]^T$, copulas in the
Wasserstein space $(C,W_1)$, and rays in the positive projective space).

\paragraph{Stratification by combinatorial type.}
It is therefore natural to view the full moduli space as a stratified
space indexed by battery skeletons:
\[
\mathfrak M
\;\simeq\;
\bigsqcup_{S\in\mathcal S} \mathfrak M_S \;\approx\; \bigsqcup_{S\in\mathcal S}
\underbrace{[0,1]^T}_{\tau\ \text{(thresholds)}}\times
\underbrace{C}_{C_u\ \text{(copulas)}}\times
\underbrace{\mathbb P(\mathbb R_+^{d_R})}_{[r]},
\]
where each stratum $\mathfrak M_S$ is the continuous parameter space
associated with a fixed combinatorial type $S$. 

In this perspective,
the ``discrete moduli'' is the set $\mathcal S$ of skeletons, while the
``continuous moduli'' over each $S$ is given by the coordinates
$(\tau,C_u,[r])$ described above.  Passing from one skeleton to another
corresponds to degenerations of the discrete data: thresholds colliding
(e.g.\ $\tau_i=\tau_j$), families merging or splitting, or anchors
appearing and disappearing.  This is directly analogous to the familiar
picture in algebraic geometry in which a moduli space is stratified by
combinatorial types (dual graphs, incidence data), with each stratum
carrying a continuous family of parameters and the boundaries between
strata encoding degenerations of that combinatorial structure.

\section{AAI as a Functional on Representations}\label{sec:aai}
\begin{definition}[AAI functional]\label{def:aai}
For each battery $\mathcal B$, an AAI functional is a measurable map $\Phi_\mathcal B:\mathcal{P}(X_\mathcal B)\to\RR$ assigning $\mathrm{AAI}_\mathcal B(\mathcal A)=\Phi_\mathcal B(\rho_\mathcal B(\mathcal A))$ and satisfying axioms:
\end{definition}

\begin{enumerate}[label=(A\arabic*)]
\item \textbf{Naturality.} If $f$ is a symmetry with pushforward $f_*:\mathcal{P}(X_\mathcal B)\to\mathcal{P}(X_{\mathcal B'})$, then $\Phi_{\mathcal B'}(f_*\nu)=\Phi_\mathcal B(\nu)$.
\item \textbf{Restricted Monotonicity.} 
If $\nu'$ dominates $\nu$ in increasing concave order (second-order stochastic dominance) with respect to the success indicators, and
\[
\mathrm{Var}_k\!\bigl(\EE_{\nu'}[\overline q(F_k)]\bigr) \le \mathrm{Var}_k\!\bigl(\EE_{\nu}[\overline q(F_k)]\bigr)
\]
(i.e., dispersion does not increase), and the expected resource cost is non-increasing, then $\Phi_{\mathcal B}(\nu')\ge\Phi_{\mathcal B}(\nu)$.
\item \textbf{Threshold calibration.} Increasing $\PP_\nu\{q(t)\ge Q^*(t)\}$ increases $\Phi_\mathcal B(\nu)$, with highest sensitivity near thresholds.
\item \textbf{Generality.} Family means enter symmetrically; dispersion penalties discourage unfair concentration.
\end{enumerate}

\begin{definition}[Tractable instance]\label{def:tractable}
Let $\overline q(F_k)=|F_k|^{-1}\sum_{t\in F_k} q(t)$. Define
\begin{equation}\label{eq:auf}
\Phi_{\mathcal B}(\nu)
=
\sum_k w_k\,\EE_\nu\!\Bigg[\frac{1}{|F_k|}\sum_{t\in F_k}
  \psi_t\bigl(q(t),Q^*(t)\bigr)\Bigg]
-\lambda\,\mathrm{Var}_k\!\bigl(\EE_\nu[\overline q(F_k)]\bigr)
-\gamma\,\EE_\nu[\mathrm{Cost}(r)].
\end{equation}
\end{definition}

\begin{remark}[Absolute vs. Projective Resources]
For this specific tractable instance, we treat absolute resource usage as part of the evaluation data. Accordingly, for the analysis of this functional, we restrict the symmetry group $G$ to exclude resource rescalings, thereby defining the moduli over absolute resource vectors rather than projective rays. Alternatively, one could define the cost term as $\EE_\nu[\mathrm{Cost}([r])]$, where $\mathrm{Cost}([r])$ depends only on the projective resource ray $[r]\in\mathbb P(\mathbb R^{d_R}_+)$ (i.e., is invariant under positive rescalings $r\mapsto c r$). This would make $\Phi_{\mathcal B}$ strictly well-defined on the projective moduli space, though at the expense of ignoring absolute resource constraints.
\end{remark}

\begin{proposition}\label{prop:tractable-axioms}
Let $\Phi_{\mathcal B}$ be given by the tractable instance
\eqref{eq:auf} in Definition~\ref{def:tractable}. Assume:
\begin{enumerate}
\item For each task $t$, the map $q\mapsto \psi_t(q,Q^*(t))$ is
measurable, nondecreasing, and concave in $q$, with maximal local
sensitivity in a neighbourhood of the threshold $Q^*(t)$.
\item The cost functional $\mathrm{Cost}(r)$ is measurable and
nonincreasing along any improvement of capability in the sense of
Axiom~\textnormal{(A2)} (so that moving from $\nu$ to $\nu'$ with
``non-increasing cost'' implies
$\EE_{\nu'}[\mathrm{Cost}(r)] \le \EE_{\nu}[\mathrm{Cost}(r)]$).
\item Symmetries $f$ of $\mathcal B$ act by permuting tasks within
families and applying resource rescalings that preserve the sets
$\{F_k\}$, the weights $w_k$, the thresholds $Q^*(t)$, and the cost
structure.
\end{enumerate}
Then $\Phi_{\mathcal B}$ is an AAI functional in the sense of
Definition~\ref{def:aai} and satisfies axioms \textnormal{(A1)}-\textnormal{(A4)}.
\end{proposition}

\begin{proof}
By construction, $\Phi_{\mathcal B}$ is a measurable map
$\mathcal P(X_{\mathcal B})\to\mathbb R$, since it is obtained by applying
measurable functions ($\psi_t$, $\overline q$, $\mathrm{Cost}$) and
finite sums, expectations, and a finite variance operator to the
underlying law $\nu$. It remains to verify axioms (A1)-(A4).

\smallskip
\noindent\textbf{(A1) Naturality.}
Let $f$ be a symmetry of $\mathcal B$ with pushforward
$f_*:\mathcal P(X_{\mathcal B})\to\mathcal P(X_{\mathcal B'})$.
By assumption (iii) $f$ acts by permuting tasks within families and
applying resource rescalings that preserve the family partition
$\{F_k\}$, the weights $w_k$, the thresholds $Q^*(t)$, and the cost
structure. Hence, under $f$, the collection of random variables
$\{q(t),Q^*(t),r\}_{t\in T}$ is mapped to a relabelled copy with the
same joint law. In particular, for every $k$,
\[
\frac{1}{|F_k|}\sum_{t\in F_k}\psi_t\bigl(q(t),Q^*(t)\bigr),
\quad
\EE_\nu[\overline q(F_k)],
\quad
\mathrm{Cost}(r)
\]
are invariant in distribution under $f$, and so are any finite linear
combinations and the variance across families. Therefore
$\Phi_{\mathcal B'}(f_*\nu) = \Phi_{\mathcal B}(\nu)$, which is exactly
naturality.

\smallskip
\noindent\textbf{(A2) Monotonicity.}
Suppose that $\nu'$ dominates $\nu$ in increasing concave order with
respect to the success indicators, and that the associated resource usage
has non-increasing expected cost
$\EE_{\nu'}[\mathrm{Cost}(r)] \le \EE_{\nu}[\mathrm{Cost}(r)]$.
By definition of increasing concave order, for every bounded, increasing,
concave function $\psi$ we have
\[
\EE_{\nu'}\bigl[\psi(q(t))\bigr] \;\ge\; \EE_{\nu}\bigl[\psi(q(t))\bigr]
\quad\text{for all tasks $t$.}
\]
In particular, taking $\psi(\cdot)=\psi_t(\cdot,Q^*(t))$ and using
assumption (i) that $q\mapsto\psi_t(q,Q^*(t))$ is increasing and concave,
we obtain
\[
\EE_{\nu'}\Big[\frac{1}{|F_k|}\sum_{t\in F_k}\psi_t\bigl(q(t),Q^*(t)\bigr)\Big]
\;\ge\;
\EE_{\nu}\Big[\frac{1}{|F_k|}\sum_{t\in F_k}\psi_t\bigl(q(t),Q^*(t)\bigr)\Big]
\]
for every family $F_k$. Hence the first term in $\Phi_{\mathcal B}$ is
weakly larger under $\nu'$ than under $\nu$.

As improvements in the $q(t)$ under $\nu'$ are assumed not to worsen the
dispersion of family means, we have
\[
\mathrm{Var}_k\!\bigl(\EE_{\nu'}[\overline q(F_k)]\bigr)
\;\le\;
\mathrm{Var}_k\!\bigl(\EE_{\nu}[\overline q(F_k)]\bigr),
\]
so the dispersion penalty $-\lambda\,\mathrm{Var}_k(\cdot)$ is also
weakly larger under $\nu'$ than under $\nu$.
Finally, the cost condition
$\EE_{\nu'}[\mathrm{Cost}(r)] \le \EE_{\nu}[\mathrm{Cost}(r)]$
implies
\[
-\gamma\,\EE_{\nu'}[\mathrm{Cost}(r)]
\;\ge\;
-\gamma\,\EE_{\nu}[\mathrm{Cost}(r)].
\]
Combining these three inequalities yields
$\Phi_{\mathcal B}(\nu') \ge \Phi_{\mathcal B}(\nu)$.

\smallskip
\noindent\textbf{(A3) Threshold calibration.}
Fix a task $t$ and consider the effect of increasing
$\PP_\nu\{q(t)\ge Q^*(t)\}$ while holding the other tasks fixed.
By assumption (i), the function
$q\mapsto\psi_t(q,Q^*(t))$ is nondecreasing and has its largest local
slope in a neighbourhood of $Q^*(t)$. Thus increasing the success
probability near $Q^*(t)$ strictly increases the contribution of task
$t$ to the first term in \eqref{eq:auf}, and the marginal effect is
maximised when the current success probability is close to $Q^*(t)$.
Since the dispersion and cost terms depend on family means and on
resources, not on local perturbations of a single task at fixed cost,
their contribution is negligible for such infinitesimal changes.
Hence $\Phi_{\mathcal B}(\nu)$ is strictly increasing in
$\PP_\nu\{q(t)\ge Q^*(t)\}$, with highest sensitivity near the threshold.

\smallskip
\noindent\textbf{(A4) Generality.}
The first term in \eqref{eq:auf} is a weighted sum of familywise
averages with symmetric treatment of tasks within each family; any
permutation of tasks that preserves the family partition leaves it
unchanged. The second term penalises dispersion of the family means
$\{\EE_\nu[\overline q(F_k)]\}_k$ via the variance across $k$, thereby
discouraging unfair concentration of performance on a small subset of
families. The cost term treats resources through $\mathrm{Cost}(r)$
without privileging any particular task family. Together, these design
choices implement axiom (A4): family means enter symmetrically and the
dispersion penalty discourages unfair concentration.

\smallskip
Combining the four parts shows that the tractable functional
$\Phi_{\mathcal B}$ in \eqref{eq:auf} is an AAI functional satisfying
axioms \textnormal{(A1)}-\textnormal{(A4)} under assumptions (i)-(iii).
\end{proof}

\begin{remark}[Dual and risk-sensitive variants]
The tractable AAI functional in \eqref{eq:auf} is concave in the
success indicators and admits a standard Fenchel-type dual
representation in terms of linear scores and the concave conjugates
$\psi_t^*$ (see, e.g., \cite{rockafellar}).  One can also obtain
risk-sensitive variants by applying entropic or other convex transforms
to the scalar base score.  We do not pursue these functional-analytic
aspects here, focusing instead on the geometric and probabilistic
structure of the moduli space.
\end{remark}

\subsection{Determinacy from dense agreement on batteries and laws}\label{subsec:determinacy}

We now define the metric structure on the moduli space and prove that a regular functional is determined by its values on a countable dense subset. This provides the rigorous justification for using finite test panels to certify general intelligence.

\begin{definition}[Canonical metric on the moduli]\label{def:moduli-metric-again}
Write the canonical representative of a battery as $(C_u,\tau,[r])$, where $C_u$ is the copula of PIT scores, $\tau$ the threshold vector, and $[r]$ the resource ray.
Fix weights $\alpha,\beta,\gamma>0$ and define
\[
d_{\mathfrak M}\big((C_u,\tau,[r]),(C_u',\tau',[r]')\big)
:= \alpha\,W_1(C_u,C_u')+\beta\,\|\tau-\tau'\|_1+\gamma\,d_{\textrm{ray}}([r],[r]').
\]
Here $W_1$ is the $1$-Wasserstein distance on $[0,1]^T$, $\|\cdot\|_1$ is extended absolutely to countable $T$, and $d_{\textrm{ray}}$ is any standard projective metric on rays in $\RR_+^{d_R}$.
\end{definition}

\begin{definition}[Canonical pushforward of a law]\label{def:canon-push}
For a battery $\mathcal B$ with canonical PIT map $u$ and resource readout $r$, let
\[
S_{\mathcal B}:X_{\mathcal B}\longrightarrow [0,1]^T\times \RR_+^{d_R},\qquad
x\mapsto \big(u(t)(x)\big)_{t\in T}\ \oplus\ r(x).
\]
For $\nu\in\mathcal P(X_{\mathcal B})$, write $\mu_{\mathcal B,\nu}:=(S_{\mathcal B})_*\nu$.
\end{definition}

\begin{definition}[Pair metric]\label{def:pair-metric}
Let $\mathcal P^\sharp:=\{(\mathcal B,\nu): \nu\in\mathcal P(X_{\mathcal B})\}$.
Define
\[
d_\sharp\big((\mathcal B,\nu),(\mathcal B',\nu')\big)
:= \alpha\,W_1\big(\mu_{\mathcal B,\nu},\mu_{\mathcal B',\nu'}\big)
+ \beta\,\|\tau-\tau'\|_1
+ \gamma\,d_{\textrm{ray}}([r],[r]').
\]
\end{definition}

\begin{definition}[Simple pairs]\label{def:simple-pairs}
Let $\mathfrak M_{\textrm{simp}}\subset\mathfrak M$ be the batteries with
(i) finite $T$;
(ii) $\tau(t)\in\mathbb{Q}\cap[0,1]$ for all $t$;
(iii) $C_u$ finitely supported on a rational grid in $[0,1]^T$;
(iv) $[r]$ represented by a rational vector in $\mathbb{Q}_+^{d_R}$.
Let $\mathcal P^\sharp_{\textrm{simp}}\subset\mathcal P^\sharp$ be the pairs $(\mathcal B,\nu)$ with $\mathcal B\in\mathfrak M_{\textrm{simp}}$ and $\mu_{\mathcal B,\nu}$ finitely supported on a rational grid in $[0,1]^T\times\mathbb{Q}_+^{d_R}$.
\end{definition}

\begin{lemma}[Countable density]\label{lem:countable-dense}
$\mathfrak M_{\textrm{simp}}$ is countable and dense in $(\mathfrak M,d_{\mathfrak M})$, and $\mathcal P^\sharp_{\textrm{simp}}$ is dense in $(\mathcal P^\sharp,d_\sharp)$.
\end{lemma}
\begin{proof}
For any measure $\mu$ on the compact metric space $[0,1]^T$, there exists a sequence of measures $\mu_n$ with finite support on rational grids such that $W_1(\mu_n,\mu) \to 0$ (standard quantization, e.g. \cite{villani} Lemma 6.18). Similarly, thresholds $\tau$ and rays $[r]$ can be approximated by rational vectors. Since $d_\sharp$ is a weighted sum of these metrics, the product of dense subsets is dense.
\end{proof}

\begin{definition}[Regular families of AAI functionals]\label{def:regular-family}
A family $\{\Phi_{\mathcal B}\}_{\mathcal B}$ of AAI functionals (Definition~\ref{def:aai}) is regular if there exists $L<\infty$ such that for all pairs
\[
\big|\Phi_{\mathcal B}(\nu)-\Phi_{\mathcal B'}(\nu')\big|
\ \le\ L\,d_\sharp\big((\mathcal B,\nu),(\mathcal B',\nu')\big).
\]
\end{definition}

\begin{theorem}[Determinacy from dense agreement]\label{thm:dense-determinacy}
Let $\{\Phi_{\mathcal B}\}$ and $\{\Psi_{\mathcal B}\}$ be regular families. If they agree on $\mathcal P^\sharp_{\textrm{simp}}$, they agree everywhere.
\end{theorem}
\begin{proof}
Fix $(\mathcal B,\nu)\in\mathcal P^\sharp$ and choose a sequence $(\mathcal B_n,\nu_n)\in\mathcal P^\sharp_{\textrm{simp}}$ with $d_\sharp((\mathcal B_n,\nu_n),(\mathcal B,\nu))\to 0$ (using Lemma~\ref{lem:countable-dense}).
Regularity implies Lipschitz continuity:
\[
\bigl|\Phi_{\mathcal B}(\nu)-\Phi_{\mathcal B_n}(\nu_n)\bigr|
\le L_\Phi\,d_\sharp\bigl((\mathcal B,\nu),(\mathcal B_n,\nu_n)\bigr),
\]
and similarly for $\Psi$. Since $\Phi_{\mathcal B_n}(\nu_n)=\Psi_{\mathcal B_n}(\nu_n)$ on simple pairs, we have
\[
\bigl|\Phi_{\mathcal B}(\nu)-\Psi_{\mathcal B}(\nu)\bigr|
\le
\bigl|\Phi_{\mathcal B}(\nu)-\Phi_{\mathcal B_n}(\nu_n)\bigr|
+
\bigl|\Psi_{\mathcal B_n}(\nu_n)-\Psi_{\mathcal B}(\nu)\bigr|
\le
(L_\Phi+L_\Psi)\,d_\sharp\bigl((\mathcal B,\nu),(\mathcal B_n,\nu_n)\bigr)\to0.
\]
Thus $\Phi_{\mathcal B}(\nu)=\Psi_{\mathcal B}(\nu)$.
\end{proof}

\begin{definition}[Order Regularity]\label{def:order-regular-family}
A family is order regular if $\Phi_{\mathcal B_n}(\nu_n) \to \Phi_{\mathcal B}(\nu)$ whenever $(\mathcal B_n,\nu_n)$ is a monotone refining sequence (monotone convergence of canonical scores in convex order, thresholds, and rays).
\end{definition}

\begin{theorem}[Determinacy from monotone-class agreement]\label{thm:monotone-determinacy}
Let $\{\Phi_{\mathcal B}\}$ and $\{\Psi_{\mathcal B}\}$ be natural, bounded, order-regular families.
If $\Phi_{\mathcal B}(\nu)=\Psi_{\mathcal B}(\nu)$ for all simple pairs $(\mathcal B,\nu)$
(i.e.\ those where canonical scores are simple with respect to a finite partition),
then $\Phi_{\mathcal B}(\nu)=\Psi_{\mathcal B}(\nu)$ for all $(\mathcal B,\nu)\in\mathcal P^\sharp$.
\end{theorem}

\begin{proof}
Fix an arbitrary pair $(\mathcal B,\nu)$. We construct a monotone approximating sequence using dyadic filtrations. Let $U$ be the canonical score variable on $[0,1]^T$. Define $U_n = \EE[U \mid \mathcal{G}_n]$, where $\mathcal{G}_n$ is the dyadic partition of the hypercube at scale $2^{-n}$. We keep thresholds and resource rays fixed, so that $(\mathcal B_n,\nu_n)$
shares the same discrete structure, thresholds and rays as $(\mathcal B,\nu)$,
and only the canonical score law is replaced by that of $U_n$.

The variables $U_n$ are simple (constant on finite partition cells). By Jensen's inequality for conditional expectations, $U_n \preceq_{cx} U_{n+1}$. By the Martingale Convergence Theorem, $U_n \to U$ in $L^1$ (hence $W_1$).
Thus, any pair $(\mathcal B,\nu)$ is the limit of a monotone sequence of simple pairs $(\mathcal B_n, \nu_n)$. Since the functionals agree on the sequence and are order regular, uniqueness of limits implies $\Phi_{\mathcal B}(\nu) = \Psi_{\mathcal B}(\nu)$.
\end{proof}

\begin{remark}[Operational corollary]\label{rem:operational}
Because $\mathcal P^\sharp_{\textrm{simp}}$ is countable and dense, a regular family of AAI functionals is uniquely determined by its values on a fixed countable catalog of simple batteries. In practice, one can calibrate $\Phi$ by tabulating this panel once.
\end{remark}

\section{The AAI Score as an AAI Functional}\label{sec:scale}
Let $\mathcal X$ be axes with weights $w_x>0$ and $W=\sum_x w_x$. The AAI-Index from AAI is the weighted geometric mean $\mathcal C=(\prod_{x\in\mathcal X} x^{w_x})^{1/W}$.
Let $\pi_x: \mathcal{P}(X_\mathcal{B}) \to [0,1]$ be the functional corresponding to the operational definition of axis $x$ as defined in \cite{AAIscore}. Specifically:
\begin{itemize}
    \item For \textbf{Autonomy} ($x=A$), $\pi_A(\nu)$ computes the mean horizon-capped action count $\mathbb{E}_\nu[\min(a(t)/H, 1)]$ composed with calibration $\phi_A$.
    \item For \textbf{Generality} ($x=G$), $\pi_G(\nu)$ calculates the fraction of task families $F_i$ where the mean quality $\bar q(F_i)$ meets the family-specific threshold $\tau_i$.
    \item For \textbf{Planning} ($x=P$), $\pi_P(\nu)$ averages the effective plan depths $d(t)$ normalized by the target depth anchor $D$ for successful traces.
    \item For \textbf{Memory} ($x=M$), $\pi_M(\nu)$ aggregates the retention half-life derived from performance decay over time lags $\Delta$ and the immediate retrieval recall $\mathrm{Rec}@K$.
    \item For \textbf{Tool Economy} ($x=T$), $\pi_T(\nu)$ computes the geometric mean of tool category coverage, success under drift perturbations $\delta$, and logarithmic discovery size.
    \item For \textbf{Self-Revision} ($x=R$), $\pi_R(\nu)$ aggregates the autonomy-weighted difference-in-differences $\rho_r \Delta \mathcal{C}_r$ (capability gain over control) for self-initiated patches.
    \item For \textbf{Sociality} ($x=S$), $\pi_S(\nu)$ measures the performance lift of multi-agent configurations over the single-agent baseline, penalized by deadlock and chatter indicators.
    \item For \textbf{Embodiment} ($x=E$), $\pi_E(\nu)$ forms the geometric mean of real-world actuation reliability, safety incident rates (weighted by severity), and sim-to-real transfer agreement.
    \item For \textbf{World-Model} ($x=W$), $\pi_W(\nu)$ computes the probabilistic calibration via the Brier score of agent predictions relative to a reference baseline.
    \item For \textbf{Economics} ($x=\text{\$}$), $\pi_{\text{\$}}(\nu)$ calculates the ratio of quality-adjusted throughput ($\mathsf{TPH}_{Q^*}$) to monetary cost ($\mathsf{CPH}$).
\end{itemize}
Define
\begin{equation}\label{eq:phi-geom}
\Phi^{geom}_\mathcal B(\nu)=\exp\!\Big(\frac{1}{W}\sum_{x\in\mathcal X} w_x\, \EE_\nu[\log \pi_x]\Big).
\end{equation}
\begin{proposition}\label{prop:geom-axioms}
$\Phi^{geom}_\mathcal B$ satisfies the axioms and for $\nu=\delta_x$ equals the AAI-Index with $x$'s axis values.
\end{proposition}
\begin{proof}
Each $\log \pi_x$ increases with axis performance; the average preserves symmetry across axes; exponentiation preserves order. For deterministic $x$, \eqref{eq:phi-geom} reduces to the geometric mean.
\end{proof}

\begin{remark}[Embedding into the tractable instance]\label{rem:embed}
Take $\psi_t(q,Q^*)=\log \pi_x$ for tasks in axis $x$ and set $\lambda=\gamma=0$; grouping tasks by axes gives \eqref{eq:phi-geom}.
\end{remark}


\subsection{Regularity of the AAI score}

We now show that our axiomatic AAI functional is stable with respect to the
moduli space metric $d_\sharp$.  We first prove a general Lipschitz-regularity
result for geometric aggregators built from axis functionals, and then check
that the concrete AAI score of \cite{AAIscore} satisfies its assumptions.

\begin{theorem}[General regularity of geometric AAI scores]\label{thm:aai-regular}
Assume each axis map $\pi_x$ is $L_x$-Lipschitz (as a function of the
canonical variables) and bounded away from zero, $\pi_x \ge \varepsilon>0$.
Define the geometric aggregator
\[
\Phi^{\mathrm{geom}}_{\mathcal B}(\nu)
=
\exp\!\Big(\tfrac{1}{W}\sum_{x\in\mathcal X}w_x\,\EE_{\mu_{\mathcal B,\nu}}\big[\log \pi_x\big]\Big),
\qquad
W:=\sum_{x\in\mathcal X}w_x.
\]
Then the family $\{\Phi^{\mathrm{geom}}_{\mathcal B}\}$ is Lipschitz-regular
with respect to $d_\sharp$.
\end{theorem}

\begin{proof}
Since $\pi_x\in[\varepsilon,1]$, the map $u\mapsto \log u$ has derivative
bounded by $1/\varepsilon$ on $[\varepsilon,1]$.
By the chain rule, the composition $\log\pi_x$ is $(L_x/\varepsilon)$-Lipschitz
in the canonical variables.  By Kantorovich--Rubinstein duality, we obtain
\[
\bigl|\E_{\mu_{\mathcal B,\nu}}[\log \pi_x]
      -\E_{\mu_{\mathcal B',\nu'}}[\log \pi_x]\bigr|
\;\le\;
\frac{L_x}{\varepsilon}\,
W_1\bigl(\mu_{\mathcal B,\nu},\mu_{\mathcal B',\nu'}\bigr).
\]
Multiplying by $w_x$ and averaging over $x$ gives
\[
\bigl|\Phi^{\mathrm{geom}}_{\mathcal B}(\nu)
      -\Phi^{\mathrm{geom}}_{\mathcal B'}(\nu')\bigr|
\le
\frac{1}{W}\sum_{x\in\mathcal X} w_x\frac{L_x}{\varepsilon}\,
W_1\bigl(\mu_{\mathcal B,\nu},\mu_{\mathcal B',\nu'}\bigr).
\]
By definition of $d_\sharp$, there exists a constant $\alpha>0$ such that
\[
W_1\bigl(\mu_{\mathcal B,\nu},\mu_{\mathcal B',\nu'}\bigr)
\;\le\;
\alpha^{-1}\,d_\sharp\bigl((\mathcal B,\nu),(\mathcal B',\nu')\bigr).
\]
Combining the two displays yields
\[
\bigl|\Phi^{\mathrm{geom}}_{\mathcal B}(\nu)
      -\Phi^{\mathrm{geom}}_{\mathcal B'}(\nu')\bigr|
\le
\frac{1}{\alpha W \varepsilon}
\Bigl(\sum_{x\in\mathcal X} w_x L_x\Bigr)\,
d_\sharp\bigl((\mathcal B,\nu),(\mathcal B',\nu')\bigr),
\]
so $\{\Phi^{\mathrm{geom}}_{\mathcal B}\}$ is Lipschitz-regular with respect
to $d_\sharp$.
\end{proof}

\begin{corollary}[Regularity of the AAI score of \cite{AAIscore}]\label{cor:aai-regular-aaiscore}
Let $\Phi^{\mathrm{AAI}}_{\mathcal B}(\nu)$ denote the AAI score of
\cite{AAIscore}, given by the geometric aggregator in
Eq.~\eqref{eq:phi-geom}.  Under the design assumptions of
\cite{AAIscore} (in particular, that each axis map $\pi_x$ is a
Lipschitz function of the canonical evaluation variables and is clipped
below at $\varepsilon>0$), the family
$\{\Phi^{\mathrm{AAI}}_{\mathcal B}\}$ is Lipschitz-regular with respect
to $d_\sharp$.
\end{corollary}

\begin{proof}
By construction in \cite{AAIscore}, each axis map $\pi_x$ is implemented
as a finite composition of Lipschitz operations (affine transforms,
thresholding, min/max, and smooth aggregations) applied to the canonical
scores and resource variables.  Hence each $\pi_x$ is $L_x$-Lipschitz
for some finite $L_x$.  The clipping step
$\pi_x^{(\varepsilon)} = \max\{\varepsilon,\pi_x\}$ ensures
$\pi_x^{(\varepsilon)}\ge\varepsilon>0$ for all inputs.
Thus the hypotheses of Theorem~\ref{thm:aai-regular} are satisfied with
$\pi_x$ replaced by $\pi_x^{(\varepsilon)}$, and the geometric AAI score
of \cite{AAIscore} is Lipschitz-regular with respect to $d_\sharp$.
\end{proof}


\section{Cognitive Cores and the Score \texorpdfstring{$\mathrm{AAI_{\textrm{core}}}$}{AAI$_{\textrm{core}}$}}\label{sec:core}

In \cite{AAIscore} a concrete
$\text{AAI}_{\text{core}}$ score is defined following \cite{hendrycks-agi} by equal-weight aggregation
over a fixed collection of CHC-style domains
$\mathcal{C}\mathcal{C}=\{\text{Gc},\text{Grw},\dots,\text{Gwm},\text{Gls},\text{Glr}\}$,
referred to there as the \emph{cognitive core}.  In the present
framework this core score corresponds to a functional
$\Psi_{\mathcal B}$ on the core factor $X^{\textrm{core}}_{\mathcal B}$,
and the maps $p_{\mathcal{C}\mathcal{C}}$ formalize the projection from
full, interface-rich batteries to their underlying cognitive core.
The following extension theory describes all ways of lifting such a
given core score to evaluations on full batteries.

Given a measurable map $p:X\to Y$ and $\nu\in\mathcal P(X)$, we write
$p_{\#}\nu\in\mathcal P(Y)$ for the pushforward measure:
\[
p_{\#}\nu(A) := \nu(p^{-1}(A)),\qquad A\in\mathcal B(Y).
\]

\begin{definition}[Cognitive core]\label{def:core}
A cognitive core for $\mathcal B$ is a factor
$p_{\mathcal{C}\mathcal{C}} : X_{\mathcal B} \twoheadrightarrow X^{\textrm{core}}_{\mathcal B}$
with sigma-algebra $\mathcal C$ such that, for every task $t$, $\1\{q(t)\ge Q^*(t)\}$ is
$\mathcal C$-measurable and $\mathcal C$ is minimal with this property (up to null sets).
Write $\nu^{\textrm{core}} := p_{\mathcal{C}\mathcal{C}\,\#}\nu$.
\end{definition}

\begin{definition}[Core score]\label{def:corescore}
Let $p_{\mathcal{C}\mathcal{C}}:X_{\mathcal B}\to X^{\textrm{core}}_{\mathcal B}$ be the core projection and
$\nu^{\textrm{core}}:=p_{\mathcal{C}\mathcal{C}\,\#}\nu$ the induced core law.
Fix a measurable lifting map
$\mathcal{L}: \mathcal{P}(X^{\textrm{core}}_{\mathcal B}) \to \mathcal{P}(X_{\mathcal B})$
such that $p_{\mathcal{C}\mathcal{C}\,\#}\circ \mathcal{L} = \mathrm{id}$.
Given an AAI functional $\Phi_{\mathcal B}:\mathcal P(X_{\mathcal B})\to\mathbb R$, define the
\emph{core AAI score} by
\[
\mathrm{AAI}_{\textrm{core},\mathcal B}(\nu)
\ :=\ 
\Phi_{\mathcal B}\!\big(\mathcal{L}(\nu^{\textrm{core}})\big).
\]
This definition can equivalently be viewed as evaluating the induced
core functional $\Psi_{\mathcal B}(\nu^{\textrm{core}})$ obtained by
restriction of $\Phi_{\mathcal B}$ to the image of $\mathcal L$.
\end{definition}

\begin{remark}
In applications, we typically choose $\mathcal L$ as the maximum-entropy lift consistent with the core marginals.
\end{remark}

\begin{proposition}[Identifiability of cores]\label{prop:ident}
The cognitive core is unique up to isomorphism. Specifically, if two
core maps $p_{\mathcal{C}\mathcal{C}}:X_{\mathcal B}\to
X^{\textrm{core}}_{\mathcal B}$ and
$p'_{\mathcal{C}\mathcal{C}}:X_{\mathcal B}\to
X'^{\textrm{core}}_{\mathcal B}$ generate the same threshold-indicator
sigma-algebra almost surely, and both core sigma-algebras are minimal
with this property, then there exists a measurable isomorphism
$h:X^{\textrm{core}}_{\mathcal B}\to X'^{\textrm{core}}_{\mathcal B}$
such that $p'_{\mathcal{C}\mathcal{C}} = h \circ p_{\mathcal{C}\mathcal{C}}$
almost surely.
\end{proposition}

\begin{proof}
By definition of the core (Definition~\ref{def:core}), the sigma-algebra
$\mathcal C:=\sigma(p_{\mathcal{C}\mathcal{C}})$ is the minimal
sigma-algebra making all threshold indicators
$\mathbf{1}\{q(t)\ge Q^*(t)\}$ measurable, and similarly for
$\mathcal C':=\sigma(p'_{\mathcal{C}\mathcal{C}})$. If the two cores
generate the same threshold-indicator sigma-algebra almost surely and
are both minimal, then $\mathcal C$ and $\mathcal C'$ coincide up to
null sets. The corresponding factor spaces
$(X^{\textrm{core}}_{\mathcal B},\mathcal C)$ and
$(X'^{\textrm{core}}_{\mathcal B},\mathcal C')$ are therefore isomorphic
as measurable spaces: one may define $h$ on $X^{\textrm{core}}_{\mathcal B}$
by sending each $p_{\mathcal{C}\mathcal{C}}$-fiber to the corresponding
$p'_{\mathcal{C}\mathcal{C}}$-fiber, which is well defined up to null
sets, and standard measure-theoretic arguments show that $h$ is
measurable with measurable inverse.
By construction $p'_{\mathcal{C}\mathcal{C}} = h\circ p_{\mathcal{C}\mathcal{C}}$
almost surely.
\end{proof}

\subsection{Given Core Scores and Continuations}
We formalize the relationship between the "true" core score and the observable score using extension theory. This isolates "interface artifacts" from "core capability."

\begin{definition}[Given core score and continuations]\label{def:core-continuation}
Fix a battery $\mathcal B$ with factor $p_{\mathcal{C}\mathcal{C}}:X_{\mathcal B}\twoheadrightarrow X^{\textrm{core}}_{\mathcal B}$ and core sigma-algebra $\mathcal C$.
Let the given core score be the functional
\[
\Psi_{\mathcal B}:\mathcal P\!\big(X^{\textrm{core}}_{\mathcal B}\big)\to\RR,\qquad
\Psi_{\mathcal B}(\nu^{\textrm{core}})=\mathrm{AAI}_{\textrm{core}, \mathcal B}(\nu),
\]
where $\nu^{\textrm{core}} := p_{\mathcal{C}\mathcal{C}\,\#}\nu$.
A \emph{continuation} of $\Psi_{\mathcal B}$ is a functional $\Phi_{\mathcal B}:\mathcal P(X_{\mathcal B})\to\RR$ such that
\[
\Phi_{\mathcal B}(\nu)=\Psi_{\mathcal B}(\nu^{\textrm{core}})\quad\text{whenever $\nu$ coincides with the canonical lift, i.e., $\nu = \mathcal{L}(\nu^{\textrm{core}})$.}
\]

\end{definition}

\begin{definition}[Core sufficiency]\label{def:core-sufficiency}
A continuation $\Phi_{\mathcal B}$ is \emph{core-sufficient} if it depends only on core information, i.e.
\[
\Phi_{\mathcal B}(\nu)=\Phi_{\mathcal B}(\nu')\quad\text{whenever}\quad \nu^{\textrm{core}}=(\nu')^{\textrm{core}}.
\]
\end{definition}

\begin{lemma}[Minimal continuation]\label{lem:minimal}
The assignment
\[
\Phi^{\min}_{\mathcal B}(\nu):=\Psi_{\mathcal B}(\nu^{\textrm{core}})
\]
is a continuation. It is the unique core-sufficient continuation.
\end{lemma}

\begin{proof}
Well-defined by composition with $p_{\mathcal{C}\mathcal{C}}$; if $\nu$ is a canonical lift, then $\nu=\mathcal{L}(\nu^{\textrm{core}})$, hence $\Phi^{\min}_{\mathcal B}(\nu)=\Psi_{\mathcal B}(\nu^{\textrm{core}})$ by definition. If $\Phi_{\mathcal B}$ is core-sufficient and a continuation, then for any $\nu$, using the sufficiency property, we have $\Phi_{\mathcal B}(\nu)=\Phi_{\mathcal B}(\mathcal{L}(\nu^{\textrm{core}}))=\Psi_{\mathcal B}(\nu^{\textrm{core}})=\Phi^{\min}_{\mathcal B}(\nu)$.
\end{proof}

\begin{proposition}[Dominance of Extensions]\label{prop:core-dominance}
Let $\nu \in \mathcal{P}(X_{\mathcal B})$. If $\nu$ dominates the canonical lift $\mathcal{L}(\nu^{\mathrm{core}})$ in the sense of Axiom (A2) (i.e., $\nu$ represents a capability improvement over the baseline lift), then:
\[
\mathrm{AAI}_{\textrm{core}, \mathcal B}(\nu)
\;=\;
\Phi^{\min}_{\mathcal B}(\nu)
\;\le\;
\Phi_{\mathcal B}(\nu).
\]
Equality holds if and only if $\Phi_{\mathcal B}(\nu) = \Phi_{\mathcal B}(\mathcal{L}(\nu^{\mathrm{core}}))$, i.e., the specific fiber structure of $\nu$ yields no score advantage over the lift.
\end{proposition}

\begin{proof}
By Lemma \ref{lem:minimal}, $\Phi^{\min}_{\mathcal B}(\nu)$ is the value of the core score functional $\Psi_{\mathcal B}$ on $\nu^{\textrm{core}}$. By Definition \ref{def:corescore}, this is defined as $\Phi_{\mathcal B}(\mathcal{L}(\nu^{\textrm{core}}))$.
Since $\nu$ dominates $\mathcal{L}(\nu^{\textrm{core}})$ by hypothesis, the Restricted Monotonicity axiom (A2) implies $\Phi_{\mathcal B}(\nu) \ge \Phi_{\mathcal B}(\mathcal{L}(\nu^{\textrm{core}}))$.
\end{proof}

\begin{definition}[Non-core invariants]\label{def:noncore-invariants}
A measurable random variable $Z:\,X_{\mathcal B}\to\RR$ is a \emph{non-core invariant} if
\[
\EE[Z\mid\mathcal C]=0\quad\text{and}\quad Z\ \text{is invariant under the evaluation-preserving symmetry group $G$}.
\]
Write $\mathcal I_{\textrm{nc}}(\mathcal B)$ for the linear space of such invariants with $\EE|Z|<\infty$.
\end{definition}

\begin{theorem}[Additive decomposition of continuations]\label{thm:decomposition}
Any continuation $\Phi_{\mathcal B}$ that is natural with respect to $G$ can be written as
\[
\Phi_{\mathcal B}(\nu)=\Phi^{\min}_{\mathcal B}(\nu)+F_{\mathcal B}(\nu),
\]
where $F_{\mathcal B}$ is $G$-invariant and vanishes on canonical lifts:
$F_{\mathcal B}(\nu)=0$ whenever $\nu=\mathcal{L}(\nu^{\textrm{core}})$.
Conversely, any such $F_{\mathcal B}$ yields a valid continuation when added to $\Phi^{\min}_{\mathcal B}$.
\end{theorem}

\begin{proof}
Define $F_{\mathcal B}(\nu):=\Phi_{\mathcal B}(\nu)-\Phi^{\min}_{\mathcal B}(\nu)$. Then $F_{\mathcal B}\equiv0$ on lifted laws by the continuation property. Naturality is inherited. The converse is immediate.
\end{proof}

\begin{proposition}[Parametric families and cardinality]\label{prop:parametric}
Fix $Z_1,\dots,Z_m\in\mathcal I_{\textrm{nc}}(\mathcal B)$. For any $\theta\in\RR^m$ set
\[
\Phi^{\theta}_{\mathcal B}(\nu):=\Phi^{\min}_{\mathcal B}(\nu)+\sum_{j=1}^m \theta_j\,\EE_\nu[Z_j].
\]
Then $\Phi^{\theta}_{\mathcal B}$ is a continuation that is natural with respect to $G$. If there exists a nonzero $Z\in\mathcal I_{\textrm{nc}}(\mathcal B)$, then the set of continuations has the cardinality of the continuum.
\end{proposition}

\begin{proof}
Since $\EE[Z_j\mid\mathcal C]=0$, for any lift $\nu=\mathcal{L}(\mu)$ we have $\EE_{\nu}[Z_j]=0$ (assuming $\mathcal{L}$ respects the conditional expectation structure, e.g., max-entropy).
Thus $\Phi^{\theta}_{\mathcal B}$ agrees with $\Phi^{\min}_{\mathcal B}$ on the core. 
$G$-invariance follows from the invariance of $Z_j$. Varying $\theta$ gives uncountably many distinct functionals when some $Z\neq0$.
\end{proof}

\begin{corollary}[Uniqueness under core sufficiency]\label{cor:unique}
If continuations are restricted to be core-sufficient, then $\Phi_{\mathcal B}=\Phi^{\min}_{\mathcal B}$ is the unique continuation.
\end{corollary}

\begin{lemma}[Calibration of finite families]\label{lem:calibration}
Let $Z_1,\dots,Z_m\in\mathcal I_{\textrm{nc}}(\mathcal B)$ and choose reference laws $\nu^{(1)},\dots,\nu^{(M)}$.
Suppose the $M\times m$ matrix with entries $[\EE_{\nu^{(i)}}(Z_j)]$ has full column rank. Then the parameters $\theta$ in $\Phi^{\theta}_{\mathcal B}$ can be identified uniquely from $M$ linear calibration conditions on $\Phi^{\theta}_{\mathcal B}(\nu^{(i)})$.
\end{lemma}

\begin{proof}
From $\Phi^{\theta}_{\mathcal B}(\nu^{(i)})-\Phi^{\min}_{\mathcal B}(\nu^{(i)})=\sum_j \theta_j \EE_{\nu^{(i)}}[Z_j]$ we obtain a full-rank linear system for $\theta$.
\end{proof}

\begin{proposition}[Lower and upper envelopes]\label{prop:envelopes}
Let $\mathcal F$ be any class of $G$-invariant $F_{\mathcal B}$ that vanish on canonical lifts and satisfy given monotonicity or resource constraints. Define
\[
\Phi^{\inf}_{\mathcal B}(\nu):=\Phi^{\min}_{\mathcal B}(\nu)+\inf_{F\in\mathcal F} F_{\mathcal B}(\nu),\qquad
\Phi^{\sup}_{\mathcal B}(\nu):=\Phi^{\min}_{\mathcal B}(\nu)+\sup_{F\in\mathcal F} F_{\mathcal B}(\nu).
\]
Then for any admissible continuation $\Phi_{\mathcal B}$ with $F_{\mathcal B}\in\mathcal F$,
\[
\Phi^{\inf}_{\mathcal B}(\nu)\le \Phi_{\mathcal B}(\nu)\le \Phi^{\sup}_{\mathcal B}(\nu)\quad\text{for all }\nu.
\]
\end{proposition}

\begin{remark}[Categorical view: right Kan extension]
Let $U:\mathbf{Bat}\to\mathbf{Core}$ send each battery to its core. The minimal continuation $\Phi^{\min}=\Psi\circ U$ is the right Kan extension of $\Psi$ along $U$. Non-core terms $F_{\mathcal B}$ are precisely $G$-invariant functionals on fibers that vanish on canonical lifts.
\end{remark}

\section{Statistical Guarantees and Stability}\label{sec:stats}

To certify that the AAI score is robust and practically estimable from finite data, we provide concentration bounds and stability guarantees under battery perturbations.

\begin{theorem}[Finite-Sample Concentration]\label{thm:concentration}
Let $\widehat{\Phi}_\mathcal B$ be the plug-in estimator of the tractable
instance \eqref{eq:auf} computed over $n$ independent seeds per task.
Assume the component functions $\psi_t$ are $L$-Lipschitz and bounded on
$[0,1]$, and that resources and weights are uniformly bounded and normalized
so that the effective per-task weights $w_k/|F_k|$ are $O(1/|T|)$. Then for
any failure probability $\delta\in(0,1)$,
\begin{equation}\label{eq:concentration}
\PP\Bigl(
\bigl|\widehat{\Phi}_\mathcal B(\rho) - \Phi_\mathcal B(\rho)\bigr|
>
L\sqrt{\frac{C\log(2/\delta)}{n|T|}} + \lambda\,\frac{K}{\sqrt{n}}
\Bigr)
\;\le\;
\delta,
\end{equation}
for constants $C,K>0$ depending only on the weights, the number of
families, and the bounds on scores and resources.
\end{theorem}

\begin{proof}
We write $\widehat{\Phi}_\mathcal B$ as the sum of a ``linear'' part
and a variance penalty:
\[
\widehat{\Phi}_\mathcal B
=
\widehat{\Phi}_{\mathrm{lin}}
-\lambda\,\widehat{\mathrm{Var}}_k\bigl(\widehat{m}_k\bigr),
\]
where
\[
\widehat{\Phi}_{\mathrm{lin}}
:=
\sum_k w_k\,\frac{1}{|F_k|}\sum_{t\in F_k}
\widehat{\EE}\bigl[\psi_t(q(t),Q^*(t))\bigr]
-\gamma\,\widehat{\EE}[\mathrm{Cost}(r)],
\qquad
\widehat{m}_k := \widehat{\EE}[\overline q(F_k)].
\]
Here $\widehat{\EE}[\cdot]$ denotes the corresponding empirical average.

\medskip\noindent
\textbf{Step 1: Linear term via McDiarmid.}
Index the $N = n|T|$ independent seed draws as $(s_{t,i})_{t\in T,\,
1\le i\le n}$ and regard $\widehat{\Phi}_{\mathrm{lin}}$ as a function
$f$ of these $N$ variables. Fix a pair $(t_0,i_0)$ and change only the
seed $s_{t_0,i_0}$. This affects only the empirical average
\[
\widehat{\EE}\bigl[\psi_{t_0}(q(t_0),Q^*(t_0))\bigr]
=
\frac{1}{n}\sum_{i=1}^n \psi_{t_0}\bigl(q(t_0;s_{t_0,i}),Q^*(t_0)\bigr),
\]
changing at most one summand. Since $q(t_0;s)\in[0,1]$ and
$\psi_{t_0}(\cdot,Q^*(t_0))$ is $L$-Lipschitz in its first argument,
the difference in that summand is at most $L$, so the empirical average
changes by at most $L/n$.

In the global functional $\widehat{\Phi}_{\mathrm{lin}}$, this term is
weighted by $w_k/|F_k|$. By the normalization assumption on the weights,
there exists a constant $C_1>0$ such that the sensitivity to a single seed is
\[
\bigl|\widehat{\Phi}_{\mathrm{lin}}(s) -
\widehat{\Phi}_{\mathrm{lin}}(s')\bigr|
\;\le\;
\frac{C_1}{n|T|}
=
\frac{C_1}{N}.
\]

Applying McDiarmid's inequality with sensitivity $\Delta = C_1/N$:
The sum of squared differences is $\sum_{i=1}^N \Delta^2 = N(C_1/N)^2 = C_1^2/N$.
Thus, for any $\eps>0$:
\[
\PP\bigl(|\widehat{\Phi}_{\mathrm{lin}} - \E\widehat{\Phi}_{\mathrm{lin}}|
> \eps\bigr)
\le
2\exp\!\left(
-\frac{2\eps^2}{C_1^2/N}
\right)
=
2\exp\!\left(
-\frac{2N\eps^2}{C_1^2}
\right).
\]
Setting the right-hand side equal to $\delta/2$ and solving for $\eps$ gives:
\[
|\widehat{\Phi}_{\mathrm{lin}} - \E\widehat{\Phi}_{\mathrm{lin}}|
\le
\frac{C_1}{\sqrt{2N}}\sqrt{\log\frac{4}{\delta}}
=
L\sqrt{\frac{C\log(2/\delta)}{n|T|}}
\]
for a suitable constant $C$, matching the theorem statement. Since $\widehat{\Phi}_{\mathrm{lin}}$ is a plug-in estimator of sums of
expectations, we have $\E\widehat{\Phi}_{\mathrm{lin}} = \Phi_{\mathrm{lin}}$,
so the same bound holds for
$|\widehat{\Phi}_{\mathrm{lin}} - \Phi_{\mathrm{lin}}|$ with probability
at least $1-\delta/2$.

\medskip\noindent
\textbf{Step 2: Variance penalty.}
For each family $F_k$, the empirical mean $\widehat{m}_k$ is an average
of $n|F_k|$ bounded variables and therefore satisfies, by Hoeffding's
inequality,
\[
\PP\bigl(|\widehat{m}_k - m_k| > \eps\bigr)
\le
2\exp\!\bigl(-2n|F_k|\eps^2\bigr)
\le
2\exp\!\bigl(-2n\eps^2\bigr),
\qquad
m_k := \EE_\nu[\overline q(F_k)].
\]
Applying a union bound over the $K'$ families and setting the right-hand
side to $\delta/2$ yields, for
\[
\eps
=
\sqrt{\frac{1}{2n}\log\frac{4K'}{\delta}},
\]
that
\[
\PP\Bigl(\max_k |\widehat{m}_k - m_k| > \eps\Bigr) \le \frac{\delta}{2}.
\]

The variance functional
\[
v(m_1,\dots,m_{K'}) := \mathrm{Var}_k(m_k)
\]
is Lipschitz on $[0,1]^{K'}$ (with $K'$ the number of families): changing a
single coordinate $m_k$ by $\delta$ changes $v(m)$ by at most $C_2
|\delta|$ for some constant $C_2>0$. Consequently, on the event from
Hoeffding's bound,
\[
\bigl|\widehat{\mathrm{Var}}_k(\widehat{m}_k)
      -\mathrm{Var}_k(m_k)\bigr|
\le
C_2 \max_k |\widehat{m}_k - m_k|
\le
C_2\sqrt{\frac{1}{2n}\log\frac{4K'}{\delta}}.
\]
Multiplying by $\lambda$ and absorbing constants gives a bound of the
form $\lambda K/\sqrt{n}$ for some $K>0$, with probability at least
$1-\delta/2$.

\medskip\noindent
\textbf{Step 3: Combination.}
Combining the two parts, we have
\[
\bigl|\widehat{\Phi}_\mathcal B - \Phi_\mathcal B\bigr|
\le
\bigl|\widehat{\Phi}_{\mathrm{lin}} - \Phi_{\mathrm{lin}}\bigr|
+
\lambda\,
\bigl|\widehat{\mathrm{Var}}_k(\widehat{m}_k)
      -\mathrm{Var}_k(m_k)\bigr|.
\]
By Step~1, the first term is bounded by
$L\sqrt{C\log(2/\delta)/(n|T|)}$ with probability at least $1-\delta/2$.
By Step~2, the second term is bounded by $\lambda K/\sqrt{n}$ with
probability at least $1-\delta/2$. A union bound then implies that both
inequalities hold simultaneously with probability at least $1-\delta$,
which is exactly the claimed concentration bound \eqref{eq:concentration}.
\end{proof}

\begin{proposition}[Stability under drift]\label{prop:sensitivity}
Assume $\{\Phi_{\mathcal B}\}$ is a regular family with Lipschitz
modulus $L_\Phi$ with respect to $d_\sharp$.
If the score copula changes by $\varepsilon$ in 1-Wasserstein distance
and thresholds $Q^*$ shift by at most $\delta$ in the $\ell_\infty$
norm, then
\[
|\Delta \Phi_\mathcal B|
\;\le\;
L_\Phi\bigl(\alpha\,\varepsilon + \beta\,|T|\,\delta\bigr)
+ \gamma B\,|\Delta \text{scale}|,
\]
where $\alpha,\beta,\gamma$ are the weights from the definition of
$d_\sharp$, $B$ is the resource bound, and $|\Delta\text{scale}|$
denotes the induced change in the resource ray.
\end{proposition}

\begin{proof}
For any two pairs $(\mathcal B,\nu)$ and $(\mathcal B',\nu')$, regularity
(Definition~\ref{def:regular-family}) gives
\[
|\Phi_{\mathcal B}(\nu) - \Phi_{\mathcal B'}(\nu')|
\le
L_\Phi \, d_\sharp\bigl((\mathcal B,\nu), (\mathcal B',\nu')\bigr).
\]
By Definition~\ref{def:pair-metric},
\[
d_\sharp\bigl((\mathcal B,\nu), (\mathcal B',\nu')\bigr)
=
\alpha\,W_1(C_u,C_u')
+ \beta\,\|\tau-\tau'\|_1
+ \gamma\,d_{\textrm{ray}}([r],[r']).
\]
If the copulas differ by at most $\varepsilon$ in $W_1$,
$W_1(C_u,C_u')\le\varepsilon$. If the thresholds shift by at most
$\delta$ in $\ell_\infty$, then $\|\tau-\tau'\|_1\le |T|\delta$.
Finally, $d_{\textrm{ray}}([r],[r'])$ is controlled by the change in
resource scale; under the resource bound $B$ this yields
$d_{\textrm{ray}}([r],[r'])\le B\,|\Delta\text{scale}|$. Substituting
these bounds into the expression for $d_\sharp$ and applying the
Lipschitz estimate gives the stated inequality.
\end{proof}

\section{Operational Implementation}\label{sec:impl}

This section translates the theoretical framework into concrete evaluation protocols, computational procedures, and extension mechanisms.

\subsection{Estimation Procedures}\label{subsec:recipes}

\paragraph{Procedure 1: The AAI-Index ($\Phi^{geom}$).}
To estimate the geometric composite score:
\begin{enumerate}
    \item For each axis $x \in \mathcal{X}$, compute the empirical expectation $\widehat{\EE}[\log \pi_x]$ over sampled seeds and drifts.
\item Aggregate via the exponential map:
\[
\widehat{\Phi}^{\mathrm{geom}}_{\mathcal B}
= \exp\Big(\frac{1}{W}\sum_{x\in\mathcal X} w_x \widehat{\EE}[\log \pi_x]\Big).
\]
    \item \textbf{Uncertainty:} Report a bootstrap confidence interval by resampling tasks within families and seeds within tasks.
\end{enumerate}

\paragraph{Procedure 2: The Tractable Functional ($\Phi_{\mathcal{B}}$).}
To estimate the generalized functional (Eq.~\ref{eq:auf}):
\begin{enumerate}
    \item \textbf{Score:} For each task $t$, estimate success $q(t)$ and compute the non-linear utility $\psi_t(q(t), Q^*(t))$.
\item \textbf{Fairness:} Compute family means
$\widehat{\overline q}(F_k)$ and subtract the dispersion penalty
$\lambda\cdot\widehat{\mathrm{Var}}_k\bigl(\widehat{\overline q}(F_k)\bigr)$.
    \item \textbf{Cost:} Subtract the resource penalty $\gamma \cdot \widehat{\EE}[\mathrm{Cost}(r)]$ using a fixed cost model.
\end{enumerate}

\paragraph{Procedure 3: Copula Estimation.}
To analyze task correlations (required for defining the moduli metric):
\begin{enumerate}
    \item Transform raw scores to $u(t)$ via Randomized PIT (Def.~\ref{def:PIT}).
    \item Estimate the dependency structure using rank-based empirical copulas or vine-copula estimators.
    \item Validate using Goodness-of-Fit tests on pairwise margins.
\end{enumerate}

\subsection{Drift and Recalibration Protocol}\label{subsec:drift}
To maintain the "don't obsess over tests" guarantee (Theorem~\ref{thm:dense-determinacy}), the evaluation must cover the moduli space locally:
\begin{itemize}
    \item \textbf{Robustness Region:} Model allowable test variance as a ball of radius $\rho$ in the copula space (Wasserstein metric) and a box of width $\delta$ around thresholds. Report the \emph{worst-case} $\Phi$ over this region.
    \item \textbf{Anchor Maintenance:} Trigger a battery recalibration when the rank-concordance of anchor agents falls below a pre-registered threshold. Freeze new anchors only during specific leaderboard windows.
\end{itemize}

\subsection{Complexity}\label{subsec:complexity}
For a battery with $|T|=m$ tasks and $n$ seeds per task:
\begin{itemize}
    \item Naive evaluation of $\Phi$ is $O(mn)$.
    \item Bootstrap Confidence Intervals scale linearly with the number of resamples $B$: $O(Bmn)$.
    \item Copula estimation scales as $O(mn \log n)$ for empirical ranks, or higher for parametric fitting.
\end{itemize}
Reproducibility requires publishing all seeds, anchor identities, and the exact cost model $\mathsf{R}$.

\section{Extensions}\label{sec:extensions}

\paragraph{Multi-Agent Games.}
The framework extends to $N$ agents evaluated on shared resources. If tasks interfere (coordination or congestion), outcomes depend on joint actions. Provided each agent's evaluation functional $\Phi_\mathcal B^i$ is
continuous in the joint mixed strategy and quasi-concave in the
agent's own mixed strategy, and strategy sets are compact and convex,
Glicksberg's theorem guarantees the existence of a mixed-strategy Nash
equilibrium in the induced evaluation game.

\paragraph{Coalitional Compositionality.}
Coalition scores can be derived via Minkowski sums of resource sets and convolutions of success laws. The dispersion penalties in our functional imply superadditivity for complementary coalitions (specialists combining to reduce variance) and subadditivity for redundant ones.

\paragraph{Dynamics of self-improvement $\kappa$.}
The AAI score $\Phi_{\mathcal B}(\mathcal A)$ measures the static
capability of an agent instance on a fixed battery. Following
\cite{AAIscore}, we can model self-improvement as a path of laws
$(\nu_r)_{r\ge 0}$ indexed by cumulative resource $r$
(e.g.\ training compute, code-rewrite budget), with
$\nu_0 = \rho_{\mathcal B}(\mathcal A)$ and
$\nu_{r+\Delta r} = T_{\Delta r}\nu_r$ for a family of transition
operators $T_{\Delta r}:\mathcal P(X_{\mathcal B})\to\mathcal P(X_{\mathcal B})$.
The absolute self-improvement rate at $\mathcal A$ is then
\[
\kappa_{\mathrm{abs}}(\mathcal A)
:=
\limsup_{\Delta r\downarrow 0}
\frac{\Phi_{\mathcal B}(T_{\Delta r}\nu_0)-\Phi_{\mathcal B}(\nu_0)}{\Delta r},
\]
and, whenever $\Phi_{\mathcal B}(\nu_0)>0$, the relative rate
\[
\kappa_{\mathrm{rel}}(\mathcal A)
:=
\frac{\kappa_{\mathrm{abs}}(\mathcal A)}{\Phi_{\mathcal B}(\nu_0)}
\]
measures proportional gain per unit resource.
Positive $\kappa_{\mathrm{abs}}$ or $\kappa_{\mathrm{rel}}$ indicates
that the agent can, on average, increase its AAI score by investing
additional resource.
In the moduli picture (Section~\ref{sec:topology}), the path
$r\mapsto (\mathcal B,\nu_r)$ traces a trajectory in $\mathcal P^\sharp$.
Under our regularity assumptions, $\Phi_{\mathcal B}$ is Lipschitz with
respect to $d_\sharp$, so $\kappa_{\mathrm{abs}}$ controls the tangent
velocity of this trajectory.
When self-improvement primarily increases success probabilities at
threshold (i.e.\ raises $\PP_{\nu_r}(q(t)\ge Q^*(t))$), the dominant
contribution to this velocity comes from the Wasserstein motion of the
copula component in the direction that increases threshold-aligned mass.


\begin{thebibliography}{1}
\bibitem{AAIscore} P. Chojecki. 
\newblock \emph{An Operational Kardashev-Style Scale for Autonomous AI - Towards AGI and Superintelligence}.
\newblock arXiv preprint arXiv:2511.13411.

\bibitem{hendrycks-agi}
D.~Hendrycks \emph{et al.}
\newblock \emph{A Definition of AGI}.
\newblock arXiv preprint arXiv:2510.18212.

\bibitem{rockafellar} R.~T.~Rockafellar and R.~J.~B.~Wets,
\emph{Variational Analysis}, Springer, 1998.

\bibitem{ruschendorf09} L.~Rüschendorf,
\emph{On the Distributional Transform, Sklar’s Theorem, and the Empirical Copula Process}, J. Stat. Plan. Inference, 2009.

\bibitem{villani} C.~Villani,
\emph{Optimal Transport: Old and New}, Springer, 2009.


\end{thebibliography}
\end{document}